\documentclass{article}

\usepackage{microtype}
\usepackage{graphicx}
\usepackage{subcaption, siunitx, graphicx, caption,tabularx}
\usepackage{booktabs} 

\usepackage{amsmath}
\usepackage{bm}
\usepackage{amsthm}
\usepackage{amssymb}

\newtheorem{definition}{Def.}
\newtheorem{assumption}{Assumption}
\newtheorem{thm}{Theorem}
\newtheorem{lem}{Lemma}
\newtheorem{fact}{Fact}

\usepackage{algorithm}
\usepackage[noend]{algorithmic}
\usepackage{xcolor}
\usepackage{dsfont}

\renewcommand{\algorithmiccomment}[1]{\hfill {\color{Gray} \footnotesize\ttfamily // #1}}

\usepackage{hyperref}

\usepackage[accepted]{icml2022}[nonatbib]

\icmltitlerunning{Efficient Model-based Multi-agent Reinforcement Learning via Optimistic Equilibrium Computation}

\begin{document}

\twocolumn[
\icmltitle{Efficient Model-based  Multi-agent  Reinforcement Learning\\ via Optimistic Equilibrium Computation}



\icmlsetsymbol{equal}{*}

\begin{icmlauthorlist}
\icmlauthor{Pier Giuseppe Sessa}{eth}
\icmlauthor{Maryam Kamgarpour}{equal,epfl}
\icmlauthor{Andreas Krause}{equal,eth}
\end{icmlauthorlist}

\icmlaffiliation{eth}{ETH Z\"urich, Rämistrasse 101, 8092 Z\"urich.}
\icmlaffiliation{epfl}{EPFL Lausanne, Rte Cantonale, 1015 Lausanne}

\icmlcorrespondingauthor{Pier Giuseppe Sessa}{sessap@ethz.ch}
\icmlkeywords{Machine Learning, ICML}

\vskip 0.3in
]



\printAffiliationsAndNotice{\icmlEqualContribution} 

\begin{abstract}
We consider model-based multi-agent reinforcement learning, where the environment transition model is unknown and can only be learned via expensive interactions with the environment. We propose \textsc{H-MARL} (Hallucinated Multi-Agent Reinforcement Learning), a novel sample-efficient algorithm that can efficiently balance \emph{exploration}, i.e., learning about the environment, and \emph{exploitation}, i.e., achieve good equilibrium performance in the underlying general-sum Markov game. \textsc{H-MARL} builds high-probability confidence intervals around the unknown transition model and sequentially updates them based on newly observed data. Using these, it constructs an \emph{optimistic} \emph{hallucinated game} for the agents for which equilibrium policies are computed at each round.
We consider general statistical models (e.g., Gaussian processes, deep ensembles, etc.) and policy classes (e.g., deep neural networks), and theoretically analyze our approach by bounding the agents' \emph{dynamic regret}. Moreover,  we provide a convergence rate to the equilibria of the underlying Markov game. We demonstrate our approach experimentally on an autonomous driving simulation benchmark. \textsc{H-MARL} learns successful equilibrium policies after a few interactions with the environment and can significantly improve the performance compared to non-optimistic exploration methods.\looseness=-1
\end{abstract}

\section{Introduction}
\label{sec:intro}
Multi-Agent Reinforcement Learning (MA-RL) has shown promising successes in solving complex sequential decision-making tasks faced by multiple interacting agents, such as in robotics~\cite{levine2016end} and game playing~\cite{mnih2015human}. 
However, its applicability to several real-world problems is still limited by the required large amount of training data. MA-RL methods can be classified as \emph{model-free}, where agents are trained directly on the obtained rewards, and \emph{model-based} where agents are trained based on an estimated model of the environment. For comprehensive surveys, see ~\citep{busoniu2008comprehensive,gronauer2021multi}.\looseness=-1

A central challenge for MA-RL, which is crucial for its scalability to the real world, is the problem of trading off \textit{exploration} with \textit{exploitation}~\cite{busoniu2008comprehensive}. Exploration is the ability to learn about the environment and generalize over unseen states and actions, hence being important to avoid suboptimal policies and enable task generalization. Exploitation of the data observed so far, on the other hand, is necessary to ensure that the agents' policies achieve high performance throughout the learning. \looseness=-1

While there has been an extensive set of techniques addressing this challenge in single-agent RL (see, e.g., \citet{jaksch2010near,luo2018,curi2020} and references therein), it remains fairly unexplored in the MA-RL domain. In MA-RL, this is particularly difficult since the joint action space grows {\em exponentially} with the number of agents and, as a result, effective single-agent RL techniques (such as $\epsilon$-greedy) can provably fail~\citet{Mahajan2019MAVENMV}. Recently, new methods have been proposed to circumvent these challenges and encourage exploration in different game settings, as discussed in Section~\ref{sec:related_works}. However, they are typically concerned with agents' asymptotic performance (i.e., neglecting agents' performance during learning), specific game settings (cooperative or two-player zero-sum), and tabular domains (i.e., finite state and action spaces).\looseness=-1

In this work, we propose a novel sample-efficient algorithm for MA-RL which can efficiently balance exploration and exploitation. We consider {\em general-sum Markov games} (also known as mixed cooperative-competitive setting in the MA-RL literature) with {\em continuous action and state spaces}, and  quantify agents' performance with the notion of \emph{dynamic regret}. The dynamic regret measures the cumulative distance (throughout the learning rounds) of the played games from being at equilibrium. We provide a regret bound for our method, together with sample-efficient convergence rates to the equilibria of the Markov game.
To the best of our knowledge, our results are the first guarantees in such a setting. We further illustrate our approach on an autonomous driving simulation benchmark.

\subsection{Related Work}\label{sec:related_works}
\vspace{-0.4em}

A set of different techniques have been proposed to encourage exploration in \emph{cooperative} MA-RL, i.e, where groups of agents share common team-level objectives.
Among those, \citet{Mahajan2019MAVENMV} enforce agents' exploration using a latent variable controlled by a hierarchical policy, \citet{wang2019influence} use influence-based techniques, \citet{liu2021cooperative} proposes a normalized entropy-based method, while \citet{zheng2021episodic} achieve exploration by promoting agents' curiosity. Moreover, \citet{Mahajan2021TesseractTA} and \citet{van2021model} develop model-based algorithms which utilize low-rank tensor decompositions of rewards and the transition function. All of these techniques, however, are not applicable to our setting of general-sum Markov games where each agent is concerned with its own reward function. Moreover, while demonstrating good experimental performance, these approaches lack theoretical guarantees or consider only asymptotic convergence.

A theoretically grounded approach to guide exploration, which has been extensively studied in \textit{single-agent} RL, is the celebrated \textit{optimism in the face of uncertainty} (OFU) principle. In a nutshell, it consists of choosing actions that maximize an optimistic estimate of the agent's value function. OFU can efficiently balance exploration with exploitation and has been applied in several single-agent RL domains  \citep[e.g.,][]{jaksch2010near, luo2018, curi2020}, yielding sample-efficient regret guarantees.
Inspired by this line of work, our approach utilizes the OFU principle in our multi-agent domain. In particular, we utilize the model-based techniques of \citet{curi2020} to compute optimistic agents' \textit{equilibria} (as opposed to single-agent optimal policies), and generalize the obtained guarantees to our much more complex MA-RL setting.

Applications of the OFU principle in MA-RL are fairly unexplored, with a few recent exceptions. The line of works by
\citet{bai2020provable}, \citet{xie2020learning}, \citet{bai2020b}, \citet{jin2021power}, \citet{loftin2021strategically} and \citet{chen2022almost} propose optimistic approaches which, however, consider only the special case of  \textit{two-player zero-sum} Markov games. \citet{pasztor2021efficient}, instead, study the problem of optimistic mean-field control. 
In our $N$ players general-sum setting, the first guarantees are by \citet{liu2021sharp} who propose a centralized optimistic value iteration scheme by adding bonus terms to the estimated Q-functions. Their algorithm is applicable only to the \textit{tabular} case (i.e., finite number of state $S$ and actions $A^i$ for each agent~$i$) and enjoys a sample-complexity guarantee of $\tilde{\mathcal{O}}(H^4S^2 \Pi_{i=1}^N A^i/ \epsilon^2)$ to reach $\epsilon$ equilibria, where $H$ is the game horizon. In the same setting, \citet{mao2022provably} shows that this can be improved to $\tilde{\mathcal{O}}(H^6S \max_{i}A^i/\epsilon^2)$ by a fully decentralized scheme. Compared to these works, we consider Markov games with \textit{continuous} states and actions, a setting where the aforementioned methods do not apply. 
Similar to \citet{liu2021sharp}, our approach follows under the `centralized training with decentralized execution' paradigm~\cite{lowe2017}. However, differently from \citet{liu2021sharp}, we construct optimistic value functions using general statistical models for the environment transition (such as Gaussian processes or deep ensembles), which allow exploiting the correlations in the transition function over continuous domains. 
Our guarantees
capture the degrees of freedom in the environment transition function and the generalization ability of the used statistical model.

\subsection{Our contributions}
\vspace{-0.5em}

We propose \textsc{H-MARL} (Hallucinated MA-RL), a novel sample-efficient algorithm for model-based MA-RL which can efficiently guide exploration by hallucinating optimistic value functions for the agents.
Regressing on past observed data, \textsc{H-MARL} builds confidence estimates around the true transition function using general statistical models. This allows exploiting correlations in the environment transition and generalizing for unseen game outcomes. Using these estimates, at each round \textsc{H-MARL} constructs upper confidence bounds on the agents' value functions and utilizes an equilibrium-finding subroutine to compute the respective agents' equilibria.
We theoretically analyze our approach by bounding the agents' dynamic regret and providing a sample-complexity guarantee to reach equilibria of the underlying general-sum Markov game.
To the best of our knowledge, ours are the first guarantees of this sort for continuous state and action spaces. We demonstrate our approach on an autonomous driving MA-RL benchmark, where \textsc{H-MARL} can quickly learn successful equilibrium policies and leads to superior performance compared to non-optimistic exploration methods.\looseness=-1

\section{Problem Setup}
We consider a Multi-Agent Reinforcement Learning (MA-RL) problem, formulated as a stochastic (Markov) game~\cite{Shapley1095} among $N$ agents over a finite episode of $H$ steps. 
At each time step $h$, the environment's state is $s_h \in \mathcal{S}\subseteq \mathbb{R}^p$ and each agent~$i$ selects action $a_h^i \in \mathcal{A}^i \subseteq \mathbb{R}^q$. Then, each agent obtains reward $r^i(s_h, a_h^1, \ldots, a_h^N)$ according to her reward function $r^i: \mathcal{S} \times \Pi_{i=1}^N \mathcal{A}^i \rightarrow \mathbb{R}$, and the environment transitions to state $s_{h+1}\sim P(\cdot |s_h, a_h^1, \ldots, a_h^N)$ where $P$ is the transition probability function. 
Each agent plays according to a policy $ \pi^i: \mathcal{S} \rightarrow \mathcal{A}^i$ which maps states to actions (our results extend also to the partially observable case, where agents have access only to a subset of the state), with the goal of maximizing her value function:
\begin{align*}
   V^i(\pi^i, \pi^{-i}) = \mathop{\mathbb{E}}\left[ \sum_{h=0}^{H-1} r^i(s_h, \pi^1(s_h), \ldots, \pi^N(s_h)) 
   \right] \,,
\end{align*}
where we have used the notation $\pi^{-i}$ to indicate the policies of all agents except agent $i$. For simplicity, we assume that the initial state $s_0$ and agent policies are deterministic, but our results can be naturally extended to the stochastic case. We let $\Pi^i$ be the policy space of agent~$i$ (e.g., neural network policies,~\citet{foerster2016learning}, Gaussian policies,~\citet{duan2016benchmarking}, etc.), and let $\bm{\Pi}:= \Pi^1 \times \cdots \times \Pi^N$ be the joint policy space. Similarly, we define $\bm{\mathcal{A}}:= \mathcal{A}^1 \times \cdots \times \mathcal{A}^N$ and let $\bm{\pi} \in \bm{\Pi}$ and $\bm{a}\in \bm{\mathcal{A}} $ be joint policy and action profiles, respectively. We also define $[N] := \{ 1,\ldots,N\}$.

We make no assumption (e.g., cooperative or zero-sum) on the game rewards structure, and consider the most general class of general-sum Markov games, also known as mixed cooperative-competitive setting in MA-RL literature. In such a setting, natural solution concepts are game \emph{equilibria}, i.e, outcomes from which rational agents' do not have incentives to deviate. We consider the most general class of equilibria, denoted as Coarse-Correlated Equilibria (CCE)~\cite{moulin1978strategically} defined as follows.

\begin{definition}[CCE]\label{def:CCE} A Coarse Correlated Equilibrium (CCE) is a distribution $\bm{\mathcal{P}}_\star$ over $\bm{\Pi}$ such that, for each agent $i$ and any policy $\pi^i \in \Pi^i$, 
\begin{equation*}
    \mathbb{E}_{\bm{\pi} \sim \bm{\mathcal{P}}_\star}[V^i(\bm{\pi})] \geq \mathbb{E}_{\pi^{-i} \sim \bm{\mathcal{P}}_\star^{-i}}[V^i(\pi^i, \pi^{-i})] \,.
\end{equation*} 
\end{definition}

CCEs generalize other equilibrium notions such as Nash equilibria~\cite{nash1950equilibrium}, and have received  significant interest from the learning community because they can be computed by decentralized algorithms in polynomial time~\cite{cesa2006prediction, marris2021multi,mao2022provably}. 
In practice, one can usually compute CCEs only up to some approximation factor. We denote such outcomes as $\epsilon$-CCE, where the condition above is satisfied only up to a $\epsilon>0$ accuracy.\looseness=-1

\paragraph{Model-based MA-RL} In this work, we take a model-based approach to compute game equilibria. Namely, we assume agents' reward functions are known\footnote{The proposed approach can also be extended to unknown reward functions, but we assume them known for ease of exposition.} (often, they are suitably designed depending on the agents' goals), and the environment's state transition follows the dynamics: 
\begin{equation}\label{eq:transition}
    s_{h+1} = f(s_h, a_h^1, \ldots, a_h^N ) + w_h ,
\end{equation} where $f:\mathcal{S}\times \bm{\mathcal{A}} \rightarrow \mathcal{S}$ is the environment transition function, and $w_h$ is zero-mean sub-Gaussian noise i.i.d. over time.

Transition function $f(\cdot)$ is a-priori  \emph{unknown} and can only be learned online via sequential interactions with the environment. 
Hence, the learning protocol goes as follows. At each round $t$, agents choose policies $\bm{\pi}_t = \{\pi_t^1, \ldots, \pi_t^N\}$ (possibly via randomization), play an episode of the Markov game, observe corresponding state transitions data $\mathcal{D}_t = \{s_{h}^t, \mathbf{a}_h^t,  h=0,\ldots, H-1\}$ and use these to improve their policies for the next round.

Interacting with the environment can be very costly, hence we seek to minimize the number of interaction rounds. At the same time, however, we want agents to achieve good performance across the played Markov games. To do so, we measure the performance of agent~$i$ after $T$ rounds by its dynamic regret:
\begin{definition}[Dynamic regret]
\begin{equation*}
R^i(T) :=   \sum_{t=1}^{T} \max_{\pi \in \Pi^i} \mathbb{E}_{\pi_t^{-i}}\left[ V^i(\pi, \pi_t^{-i})\right] -  \mathbb{E}_{{\bm\pi}_t}\left[V^i({\bm\pi}_t)\right] \,.
\end{equation*}
\end{definition}
The dynamic regret measures the cumulative difference between the best achievable expected value of the game at each round $t$, and the actual value obtained. Thus, it is an indicator of how far the played games were from being at equilibrium 
(note that if $\bm{\pi}_t$ is sampled from a CCE at each round $t$, $R^i(T) = 0$ for all $i$ by Def.\ref{def:CCE}). The notion of dynamic regret is widely adopted in multi-agent learning~\citep[see, e.g.,][]{zhang2020minimizing} and represents a stronger performance benchmark than the \emph{static} regret~\cite{cesa2006prediction}, which compares the obtained rewards with the ones of the best fixed policy in hindsight. Intuitively, we can compete with such a stronger benchmark because we allow centralized training of MA-RL agents and learn to play a time-invariant Markov game repeated over rounds. 

In the next section, we propose an algorithm that can efficiently balance exploration (learning about the environment) and exploitation (achieving low regret). Provided ``sufficient model generalizability'' (as formalized later on), it ensures that the agents' regret grows sublinearly with the interaction rounds $T$, i.e., $\lim_{T\rightarrow \infty} \frac{1}{T }R^i(T) \rightarrow 0$. 

\section{The H-MARL Algorithm}\label{sec:algorithm}

We propose \textsc{H-MARL}, a novel sample-efficient algorithm for the model-based MA-RL setting defined above. The algorithm falls under the widely adopted `centralized training with decentralized execution' (CTDE) paradigm~\citep{lowe2017}, i.e, agent policies are trained centrally but, once deployed, they require only agents' local information. The proposed method utilizes two main building blocks: 1) learning the transition model from observed data, and 2) hallucinating optimistic agents' value functions. We describe them in detail next.

\subsection{Learning the transition model from data} 
At the end of each round $t$, we can use observed transition data (along with possibly available offline data) $\{\mathcal{D}_\tau\}_{\tau=1}^t$ to estimate the transition function $f$ via its posterior \emph{mean} and \emph{confidence} functions
\begin{equation*}
\mu_t(s, \mathbf{a}) \in \mathbb{R}^p,\quad  \Sigma_t(s, \mathbf{a})\in \mathbb{R}^{p\times p},
\end{equation*}
respectively. For this purpose, a variety of statistical models can be used, depending on the problem. Statistical models allow exploiting correlations in the observed data and generalize for non-visited states and joint actions. A concrete example are Gaussian Processes~\citep[GPs, cf.,][]{williams2006gaussian}. GPs are powerful non-parametric models widely used in Bayesian optimization and have recently found new application areas, e.g., in online learning~\cite{sessa19noregret,sessa20contextual} and model-based RL~\cite{curi2020, Curi2021CombiningPW}. According to the GP model, $\mu_t$ and $\Sigma_t$ are obtained via kernel-ridge regression on the observed data. Popular kernel choices include linear, squared exponential, and Màtern kernels~\cite{srinivas2009gaussian}. 
For higher dimensional state and action spaces, a more scalable alternative is represented by deep Neural Network (NN) ensembles~\cite{Lee2015WhyMH,lakshminarayanan2016simple}. In such a case, an ensemble of NNs can be trained on the data $\{\mathcal{D}_\tau\}_{\tau=1}^t$ (e.g., with bootstrapping techniques or by random shuffling the whole dataset), and posterior mean $\mu_t$ and confidence $\Sigma_t$ are computed by aggregating the ensemble predictions.

We make no assumptions on the used statistical model, except that it is (conservatively) {\em calibrated}:
\begin{assumption}\label{ass:calibrated} We assume the statistical model is calibrated, i.e., there exists 
$\beta_t \in \mathbb{R}_+$ such that with probability at least $1-\delta$, $|f(s, \mathbf{a})- \mu_t(s, \mathbf{a})| \leq  \beta_t  \sigma_t(s, \mathbf{a})$ holds coordinate-wise, $\forall s, \forall \mathbf{a}, \forall t$, where $\sigma_t(\cdot) = \text{diag}(\Sigma_t(\cdot))$.
\end{assumption}
A calibrated model allows us to bound, with a given confidence, how the trajectories predicted by the learned model plausibly differ from those corresponding to the true environment. 
We note that GP models readily satisfy Assumption~\ref{ass:calibrated}, provided that $f(\cdot)$ has a bounded norm in the reproducing kernel Hilbert space associated to the used kernel function  \citep[cf.,][]{srinivas2009gaussian, chowdhury17kernelized}. Moreover, in case of ensemble NNs models, different calibration techniques have been proposed in the literature, e.g., by \citet{kuleshov2018accurate} and \citet{zhang2020mix}.

\subsection{Hallucinating agents' value functions}
Using calibrated posterior mean and confidence $\mu_{t}(\cdot), \Sigma_{t}(\cdot)$, the proposed algorithm consists of building hallucinated \emph{optimistic} value functions for the agents. For each agent~$i$, these are obtained as:
\begin{align}
    & \textup{UCB}^i_t(\bm{\pi})  = \max_{\eta(\cdot) \in [-1,1]^p}  \mathop{\mathbb{E}}\left[\sum_{h=0}^{H-1} r^i(s_h, \mathbf{a}_h)\right] \label{eq:UCB_Vf_modelbased}\\
& \hspace{1.5em}  \text{s.t.} \quad  
 \mathbf{a}_h = \bm{\pi}(s_h) \nonumber\\ 
& \hspace{3.8em} s_{h} = \mu_{t}(s_{h-1},\mathbf{a}_{h-1}) \nonumber \\ & \hspace{4em}+ \beta_t \cdot  \Sigma_{t}(s_{h-1},\mathbf{a}_{h-1}) \eta(s_{h-1}, \mathbf{a}_{h-1}) + w_{h} \nonumber \,.
\end{align}
The function $\textup{UCB}^i_t(\cdot)$ maps a joint policy profile $\bm{\pi}$ to an optimistic estimate of the value function for agent~$i$. The optimism is injected through the auxiliary function $\eta(\cdot)$ which, for each state and agents' joint action, selects the state transition that leads to the largest expected cumulative reward, yet being plausible (by Assumption~\ref{ass:calibrated} and since $\eta(s,\mathbf{a}) \in [-1,1]^p,  \forall (s,\mathbf{a})$). More precisely, under Assumption~\ref{ass:calibrated} it holds $\textup{UCB}^i_t(\bm{\pi}) \geq V^i(\bm{\pi})$, for all joint policies $\bm{\pi}$ and all rounds $t\geq 1$, as stated in Lemma~\ref{lem:Vf_confidence_lemma} in Appendix~\ref{app:appendix_proof_thm}. Hence, the functions $\textup{UCB}^i_t(\cdot)$ are provable upper confidence bounds for the agents' value functions.
Their computational bottleneck is represented by the outer maximization over general functions $\eta(\cdot)$. To alleviate this, in Section~\ref{sec:approximation} we provide a practical approximation of $\textup{UCB}^i_t(\cdot)$ via sampling, which we also use in our experiments.
Finally, we note that Eq.~\eqref{eq:UCB_Vf_modelbased} can be viewed as the multi-agent generalization of the hallucinated value function proposed by \citet{curi2020} for single-agent RL. 
\looseness=-1
These functions are at the core of the proposed approach. Indeed, our proposed \textsc{H-MARL} algorithm utilizes an equilibrium finding subroutine to compute, at each round $t$, a CCE of the hallucinated Markov game defined by the value functions $\{ \textup{UCB}^i_{t-1}(\cdot), i=1,\ldots, N\}$. Let $\bm{\mathcal{P}}_{t}$ be the computed CCE equilibrium. Then, agents play the Markov game using equilibrium policies $\bm{\pi}_{t} \sim \bm{\mathcal{P}}_{t}$, and the transition models $\mu_t, \Sigma_t$ are updated based on the newly observed data. We summarize our overall approach in Algorithm~\ref{alg:H_MARL}.

\begin{algorithm}[t!]
\caption{The \textsc{H-MARL} algorithm}\label{alg:H_MARL}
\begin{algorithmic}[1]
     \REQUIRE Agents' policy spaces $\Pi^1, \ldots ,\Pi^N$.
    \FOR {$t = 1, \ldots, T$} {}
            \vspace{0.2em}
    \STATE  $\bm{\mathcal{P}}_{t} \leftarrow \textmd{Find-CCE}\big(\textup{UCB}^1_{t-1}(\cdot), \ldots, \textup{UCB}^N_{t-1}(\cdot) \big)$, \\
        with $\textup{UCB}^i_{t-1}(\cdot)$ defined in Eq. \eqref{eq:UCB_Vf_modelbased}. 
        \vspace{0.3em}
        \STATE Episode rollout using policies $$\bm{\pi}_{t} = (\pi_{t}^1,\ldots, \pi_{t}^N) \sim \bm{\mathcal{P}}_{t}$$
        \STATE Update transition model $\mu_t(\cdot,\cdot)$, $\Sigma_t(\cdot,\cdot)$, using observed $H$ transitions. 
    \ENDFOR  
    
\end{algorithmic}
\end{algorithm}


We leave the equilibrium computation step (Line 2 in Algorithm~\ref{alg:H_MARL}) very general, as this can be achieved by various MA-RL methods for general-sum Markov games. Importantly, because equilibrium is computed with respect to the hallucinated value functions, this step {\em does not require interacting with the true environment} and hence sample-efficiency is not crucial here. Accuracy of the returned equilibrium can be traded off with its computational complexity and different algorithms are more suitable than others, depending on the game. A list of practical methods has demonstrated good empirical performance in computing equilibrium policies, e.g., using independent learners, actor-critic formulations~\cite{iqbal2019actor}, policy gradients~\cite{lowe2017}, or exploiting mean-field approximations~\cite{yang18d}. Provably convergent approaches also exist, e.g., using optimistic~\cite{liu2021sharp} or decentralized~\cite{mao2022provably} Q-learning. For ease of exposition we assume an exact CCE is computed at each round, but we will also discuss the case where only $\epsilon$-CCEs are obtained.


\subsection{Theoretical guarantees}
We now present theoretical guarantees for \textsc{H-MARL}. More specifically, we obtain a dynamic regret bound for the agents after $T$ rounds. Additionally, we provide an offline sample-complexity bound on the number of rounds $T$ to reach an $\epsilon$-CCE of the underlying Markov game.

Our guarantees depend on the following quantity, which characterizes the complexity of the transition model:
\begin{equation}\label{eq:complexity_measure}
\mathcal{I}_T := \max_{\substack{\mathcal{D}_1, \ldots ,\mathcal{D}_T \\ \mathcal{D}_i \subset \mathcal{S} \times \bm{\mathcal{A}} \\  |\mathcal{D}_i| = H}}  \quad  \sum_{t=1}^{T} \sum_{(s, \mathbf{a}) \in \mathcal{D}_t} \|\sigma_{t-1}(s, \mathbf{a}) \|_2^2 \,.
\end{equation}
It quantifies the maximum predictive uncertainty about the model, where the worst case is taken over all possible observed transitions up to round $T$. Intuitively, easier transition models should be learned with less uncertainty, thus leading to a smaller $\mathcal{I}_T$. \citet{curi2020} define the quantity $\mathcal{I}_T$ for the single-agent case and show that, although it is generally impossible to compute, it can be bounded for the special case of GP models. The same considerations apply to our setting: it holds $\mathcal{I}_T \leq pH\gamma_{HT}$, where $\gamma_{HT}$ is the \emph{maximum information gain} about $f$ from $HT$ noisy observations, a typical quantity in Bayesian optimization~\citep{srinivas2009gaussian,chowdhury17kernelized}. Known bounds for $\gamma_{HT}$ exist depending on the kernel, e.g., $\gamma_{HT} \leq \mathcal{O}(\log(HT)^{d+1})$ and $\gamma_{HT} \leq \mathcal{O}(d \log(HT))$, respectively for  squared-exponential and linear kernels, where $d=p+Nq$ is the domain dimension~\cite{srinivas2009gaussian}. 

Moreover, the obtained guarantees rely on the following Lipschitz assumptions.
\begin{assumption}
\label{ass:lipschitzness}
We assume the transition function $f$, reward functions $r^i$, policies $\pi \in \Pi^i$, and the posterior standard deviation function $\sigma_t$ are Lipschitz continuous w.r.t. $\|~\cdot~\|_2$ with constants $L_f$, $L_r$, $L_\pi$, and $L_\sigma$, respectively for all $i\in [N]$ and $t \geq 0$.
\end{assumption}

Lipschitz continuity of the transition function $f$ is required for generalization, otherwise if $f$ changes too abruptly we expect any efficient model-based method which aims to learn it to fail. Moreover, lipschitzess of rewards and policies is not a restrictive assumption since they are typically hand-designed. In addition, GP models lead to Lipschitz continuous posterior standard deviations, according to the kernel metric~\citep[cf.,][Lemma 13]{curi2020}.

The following main theorem provides a dynamic regret bound for \textsc{H-MARL}, as a function of the different game quantities. Its proof is relegated to Appendix~\ref{app:appendix_proof_thm}.
\begin{thm}[Dynamic regret bound]\label{thm:thm1}
Let Assumptions~\ref{ass:calibrated} and~\ref{ass:lipschitzness} be satisfied. After $T$ rounds, the \textsc{H-MARL} algorithm ensures that, with probability $1-\delta$, each agent~$i$ has bounded dynamic regret:
\begin{equation*}
R^i(T)  \leq  \bar{L} H^{1.5} \sqrt{T \mathcal{I}_T},
\end{equation*}
where $\bar{L}=\mathcal{O}\big(N^{H/2} L_\pi^{H/2} (\bar{\beta}^H L_\sigma^H + L_f^H) + \log(1/\delta) \big)$, $\bar{\beta} = \max_{t} \beta_t$, and $\mathcal{I}_T$ is the complexity measure defined in \eqref{eq:complexity_measure}.
\end{thm}

Theorem~\ref{thm:thm1} shows that the overall regrets' rate, as a function of the interaction rounds $T$, depends on the growth rate of $\mathcal{I}_T$ and hence on the generalization ability of the statistical model (if the model does not generalize at all, we expect $\mathcal{I}_T \geq \Omega(T)$ and the agents' regrets grow superlinearly). In case of GP models, one has $\bar{\beta} = \mathcal{O}(\sqrt{\gamma_{HT}})$ and $\mathcal{I}_T \leq pH\gamma_{HT}$, yielding overall regret rates of $R^i(T) \leq \mathcal{O}(N^{H/2} H^{2} \sqrt{pT}\, \gamma_{HT}\,^{(H+1)/2})$. Substituting the bounds on the maximum information gain $\gamma_{HT}$, one obtains \emph{sublinear} regrets for commonly used kernels such as the linear and squared exponential kernel. We note that the regret rates of Theorem~\ref{thm:thm1} generalize the single-agent optimization guarantee of \citet{curi2020}, with an additional multiplicative factor of $\mathcal{O}(N^{H/2})$ representing the price of dealing with a multi-agent environment. Finally, note that Theorem~\ref{thm:thm1} assumes that an exact CCE is computed at each round (Line 2 of Algorithm~\ref{alg:H_MARL}), but in Appendix~\ref{app:appendix_proof_thm} we discuss the more general case where only $\epsilon_t$-CCEs are obtained.

While Theorem~\ref{thm:thm1} concerns the agents' performance throughout the interaction rounds, the next theorem provides an \emph{offline} sample complexity guarantee (i.e., on the performance achieved after $T$ rounds) for \textsc{H-MARL}, characterizing a sufficient number of rounds $T$ to reach an $\epsilon$-CCE of the underlying Markov game. First, for each joint policy $\bm{\pi}$ and each agent $i$, we define the \emph{lower} confidence estimate $\textup{LCB}^i_t(\bm{\pi})$ as the solution of Eq.~\eqref{eq:UCB_Vf_modelbased} where the outer maximization is replaced by a minimization over $\eta(\cdot)$.
\begin{thm}[Offline performance - Convergence to $\epsilon$-CCE]\label{thm:thm2}
Let Assumptions~\ref{ass:calibrated} and~\ref{ass:lipschitzness} be satisfied and assume we run \textsc{H-MARL} for $T$ rounds. Moreover, consider distribution $\bm{\mathcal{P}}_{t^\star}$ such that 
$$t^\star = \arg \min_{t \in [T]} \max_{i\in [N]} \mathop{\mathbb{E}}_{\bm{\pi} \sim \bm{\mathcal{P}}_{t}}[ \textup{UCB}^i_{t-1}(\bm{\pi})-\textup{LCB}^i_{t-1}(\bm{\pi})] .$$
For a given accuracy $\epsilon \geq 0$, if 
$$ T \geq  \Omega\left( \frac{1}{\epsilon^2} \cdot \bar{L}^2 H^3 \mathcal{I}_T\right),$$
then $\bm{\mathcal{P}}_{t^\star}$ is an $\epsilon$-CCE of the underlying Markov game with probability at least $1-\delta$.
\end{thm}
The proof of Theorem~\ref{thm:thm2} can be found in Appendix~\ref{app:appendix_proof_thm}. Intuitively, it is obtained by showing that, after $T$ rounds, distribution $\bm{\mathcal{P}}_{t^\star}$ 
is an $\epsilon_T$-CCE of the Markov game, where the approximation factor $\epsilon_T$ is bounded as $\epsilon_T \leq  \mathcal{O}\big(T^{-\frac{1}{2}}\bar{L} H^{1.5} \sqrt{\mathcal{I}_T}\big)$. This implies the result.

\subsection{Practical implementation via sampling}\label{sec:approximation} 
Although the optimistic game values in Eq.~\eqref{eq:UCB_Vf_modelbased} are hard to compute, we propose the following more practical approximation:
\begin{align}
    \widetilde{\textup{UCB}}^i_t(\bm{\pi}) & =  \mathop{\mathbb{E}}\left[\sum_{h=0}^{H-1} r^i(s_h^\star, \mathbf{a}_h)\right] \label{eq:UCB_practical}\\
& \text{s.t.}  \quad \eta_h^j \sim \text{Unif}([-1,1]^p), \quad j=1,\ldots, Z \nonumber\\ 
& \hspace{2.3em} s_{h}^j = \mu_{t}(s^\star_{h-1},\mathbf{a}_{h-1}) \nonumber \\ & \hspace{5em}+ \beta_t \cdot  \Sigma_{t}(s^\star_{h-1},\mathbf{a}_{h-1})\cdot \eta_h^j + w_h \nonumber \\ 
& \hspace{2.3em} s_h^\star = \arg\max_{s_h^j, j\in \{1,\ldots, S\}} r^i(s_h^j, \bm{\pi}(s_h^j)) \label{eq:practical_maximiz} \\ 
&\hspace{2.3em}  \mathbf{a}_h = \bm{\pi}(s_h^\star)\,. \nonumber
\end{align}
For any given joint policy $\bm{\pi} \in \bm{\Pi}$, function $\widetilde{\textup{UCB}}^i_t(\bm{\pi})$ approximates the optimistic value function $\textup{UCB}^i_t(\bm{\pi})$ (computed as in \eqref{eq:UCB_Vf_modelbased}) via the following sampling method: At each time $h$, a finite set of $Z$ auxiliary parameters $\{\eta_h^j, j=1,\ldots, Z \}$ is sampled uniformly from $[-1,1]^p$. These are used to compute $Z$ plausible states $s_h^j, j=1,\ldots, Z$, according to the previous state, joint actions, and the hallucinated transition model. Among these, only the state $s_h^\star$ leading to the largest reward is selected and used for the next transition. This process continues up to time $t=H-1$. Further, the outer expectation can be approximated by taking the empirical mean over multiple episodes.

\looseness -1 Note that $\widetilde{\textup{UCB}}^i_t(\bm{\pi}) \leq \textup{UCB}^i_t(\bm{\pi})$ for all $\bm{\pi}$, since the computed state sequences $s^\star_h$ are feasible trajectories for the maximization problem in Eq.~\eqref{eq:UCB_Vf_modelbased}. However, Eq.~\eqref{eq:UCB_practical} offers significant computational advantages: 1) the intractable maximization over functions $\eta(\cdot)$ is replaced by selecting the best out of finitely many samples, 2) the optimization over the cumulative reward is broken down into greedily optimizing the step-wise rewards. While the first approximation can be alleviated by considering a large number of samples, the second one can be problematic in case of sparse rewards (i.e., where the optimal sequence of parameters $\eta_h^j$ depends on the whole state trajectory). In such a case, instead of greedily selecting state $s_h^\star$ as in Eq.~\eqref{eq:practical_maximiz}, one could keep track of $Z$ parallel trajectories $\{s_h^j\}_{h=0}^{H-1}, j=1,\ldots,Z$, according to the sampled auxiliary variables $\eta^j_h$ and, only after time $H$, select $\{s_h^\star\}_{h=0}^{H-1}$ to be the one leading to the largest cumulative reward. In the latter approach, however, a larger $Z$ should be selected to obtain reasonable approximations.

\begin{figure*}[ht]
\centering
\begin{subfigure}[b]{0.38\textwidth}
\includegraphics[height=.5\textwidth]{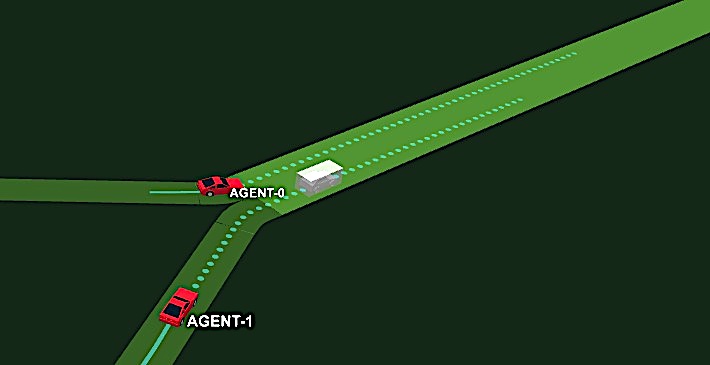}
\caption{Lane merging scenario}
\label{fig:double_merge_scenario}
\end{subfigure}
\hspace{2.5em}
\begin{subfigure}[b]{0.38\textwidth}
\includegraphics[height=.5\textwidth]{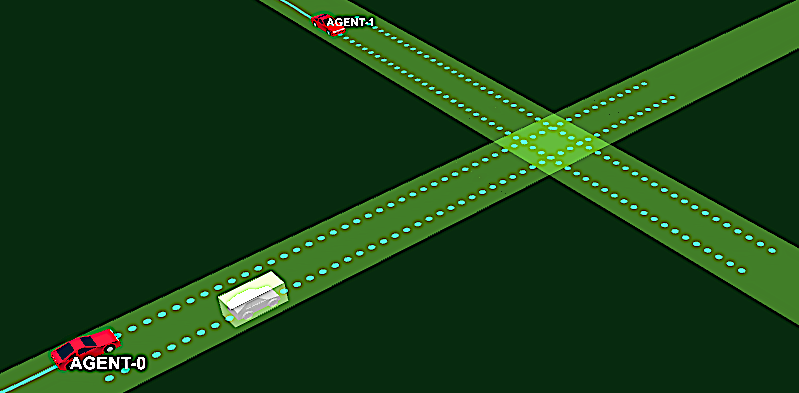}
\caption{Intersection scenario}
\label{fig:intersection_scenario}
\end{subfigure}
\vspace{-0.9em}
\caption{(a) \textbf{Lane merging scenario:} both agents want to maximize progress, while AGENT-$0$  also wants to merge into AGENT-$1$'s lane. The grey car is a human-driven (HD) vehicle with random initial speed. (b) \textbf{Intersection scenario:} AGENT-0 wants to turn right, while AGENT-1 wants to proceed straight on its lane. The HD vehicle has a random initial speed and wants to cross the intersection. }
\label{fig:scenarios}
\end{figure*}

\section{Experiments: Autonomous driving SMARTS benchmark}\label{sec:experiments}
We demonstrate our approach on an autonomous driving application. The goal is to find successful equilibrium policies for Autonomous Vehicles (AVs) when driving on common roads. Crucially, once deployed, AVs must interact with other non-controllable human-driven (HD) vehicles, which can significantly affect the AVs' performance. We make the realistic assumption that at training time the behavior of HD vehicles is a-priori {\em unknown} and can only be inferred by {\em online interactions} with the real-world, or with an expensive simulator~\citep[e.g., requiring humans-in-the-loop, cf.,][]{reddy2018shared}. This problem can be naturally formulated according to our multi-agent RL framework. Indeed, we show below that myopically considering it as a single-agent RL problem leads to significantly inferior performance. The transition function $f(\cdot)$ includes the unknown HD vehicles' behavior (e.g., how do they act at any given time, given the current environment state) and can be learned from observed trajectories from past rounds $1,\ldots, t-1$. Note that the environment transition function also includes the AV dynamics, which are assumed to be known for simplicity.

Hence, given a pre-designed driving scenario, we seek to evaluate the performance of the policies computed by \textsc{H-MARL} for the AVs after each interaction with the environment. In particular, we are interested in comparing the proposed optimistic exploration strategy with the greedy and perhaps most naive approach of using the predictive posterior mean of $f(\cdot)$ to plan for the next round. We also consider Thompson Sampling (TS)~\cite{russo2018tutorial} as a natural exploration baseline: in our setup TS consists of sampling, at each policy evaluation step, state $s_{h+1}$ from the posterior Gaussian distribution with mean $\mu_t(s_h,\mathbf{a}_h)$ and covariance $\Sigma_t(s_h,\mathbf{a}_h)$. Note that this is substantially different from the proposed optimistic \textsc{H-MARL} where, among plausible next states, we select the ones leading to the highest agents' reward. Moreover, to assess the quality of learning, we compare all these approaches with the idealized benchmark of knowing the HD model in advance. 
First, we describe the used environment, the agents' specifications, and the HD vehicles' model.\looseness=-1

\textbf{SMARTS environment.} We run our experiments using the open-source SMARTS autonomous driving platform~\cite{zhou2020smarts}, a recently proposed benchmark for MA-RL research. SMARTS provides realistic simulations for autonomous vehicles on configurable road environments, as well as interactions with background traffic vehicles, which for our scope represent HD vehicles. Vehicle dynamics are simulated by the \textit{Bullet physics engine}~\cite{coumans2021}, while SUMO~\cite{krajzewicz2002sumo} is the background traffic provider. At a higher level, SMARTS is integrated with the reinforcement learning RLlib~\cite{liang2018rllib} library, allowing us to train deep RL policies for the AVs (agents) using existing MA-RL baselines.

\textbf{Observations, policies, rewards.} Each agent is assigned a mission, which is represented by a start and a goal position. Simulation steps are of $0.1$s and, at each step, the state observed by agent~$i$ consists of: relative position to the goal, the distance to the lane's center, speed, steering, and heading errors, and the states of its neighboring vehicles. Each agent has a discrete action space: $\{ \text{keep lane}, \text{slow down}, \text{turn right}, \text{turn left}\}$ and a policy parametrized by a deep neural network with 2 hidden layers of 256 units and $\tanh$ activations (we use default policies from~\citet{zhou2020smarts}). The reward of agent~$i$ at each time $h$ is
$r^i(s_h,\mathbf{a}_h) =  r^i_\text{bonus}(s_h,\mathbf{a}_h) - r^i_\text{penalty}(s_h,\mathbf{a}_h)$,
where $r_\text{bonus}$ rewards progress (i.e., driven distance) and reaching the goal, while $r_\text{penalty}$ penalizes acceleration, sharp turns, collisions, and distances from lane center and goal.

\textbf{HD vehicles' model.} \looseness -1 At test time, HD vehicles are controlled by the SUMO traffic provider, which is a black box for our purposes. To learn their behavior, we use a GP model which maps the current HD state and its relative position to the other vehicles, to the next HD state (i.e., one-step-ahead prediction). We use a Mat\'ern kernel for predicting the speed change (as we expect it to be rather nonsmooth), and a squared exponential kernel to predict its change in position. 

\textbf{\textsc{H-MARL} implementation.} 
To compute the equilibrium policies at each round $t$ (Line 2 of Algorithm~\ref{alg:H_MARL}), we use independent Deep Q-Networks~\citep[DQN,][]{mnih2015human} and we let $\bm{\mathcal{P}}_{t+1}$ be the uniform distribution over the last 4 policy iterations (or checkpoints). This avoids non-convergence behaviors and mimics the way CCEs are computed by independent no-regret learning in normal-form games, see, e.g., \citet{cesa2006prediction}. We chose independent DQN for its computational efficiency and because it was shown by \citet{zhou2020smarts} to find good driving policies. 
The hallucinated optimistic value functions $\text{UCB}_t^i(\cdot)$ are approximated by the sampling approach of Eq.~\eqref{eq:UCB_practical} with $Z=5$ samples at each time step and $\beta_t = 1.0$. This effectively corresponds to sampling, at each policy evaluation step, $Z$ plausible HD vehicles' states and selecting the one that leads to the highest reward.
Additional details of our experimental setup are provided in Appendix~\ref{app:appendix_experiments}.\looseness=-1 

\subsection{Results}
We consider two different realistic scenarios corresponding to lane merging and intersection.

\begin{figure}[t!]
\centering
\begin{subfigure}[b]{0.23\textwidth}
\includegraphics[width=\textwidth]{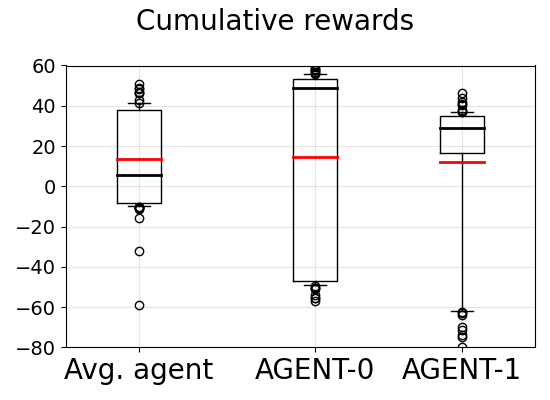}
\caption{Single-agent optima.}
\label{fig:double_merge_completion rates_1}
\end{subfigure}
\hspace{0.3em}
\begin{subfigure}[b]{0.23\textwidth}
\includegraphics[width=\textwidth]{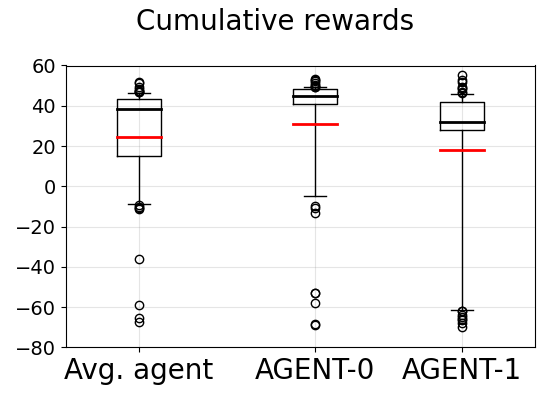}
\caption{Multi-agent equilibrium.}
\label{fig:double_merge_completion rates_2}
\end{subfigure}
\vspace{-0.3em}
\caption{\textbf{Lane merging scenario.} MA-RL equilibrium policies (b) lead to higher average and individual rewards for the agents, with respect to using single-agent optimized policies (a). Boxplots of 15-85$^{th}$ percentiles over 50 runs (mean in red, median in black).}
\label{fig:equilibrium_benefit}
\end{figure}

\begin{figure*}[t!]
\centering
\begin{subfigure}[b]{0.38\textwidth}
\includegraphics[width=\textwidth]{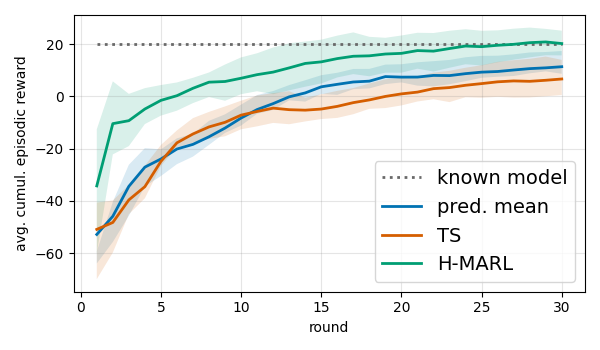}
\vskip-5pt
\caption{Lane merging scenario}
\label{fig:double_merge_cumulative_learning_1}
\end{subfigure}
\hspace{1.5em}
\begin{subfigure}[b]{0.38\textwidth}
\includegraphics[width=\textwidth]{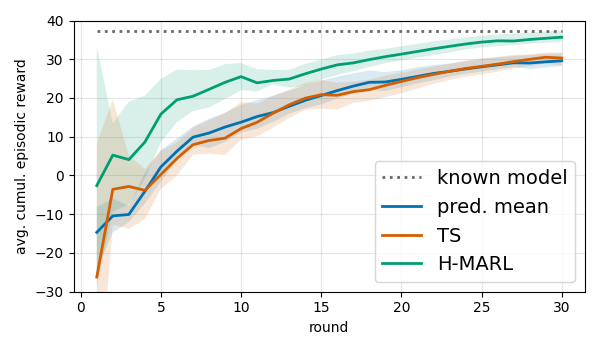}
\vskip-5pt
\caption{Intersection scenario}
\label{fig:double_merge_cumulative_learning_2}
\end{subfigure}
\vspace{-0.9em}
\caption{Average game values (mean and 30-70$^{th}$ percentiles) as a function of the interaction rounds, when equilibrium policies are computed according to the optimistic \textsc{H-MARL} algorithm, Thompson Sampling (TS) exploration strategy, or when using the predictive posterior mean about the transition function (no explicit exploration). We also compare against the idealized benchmark of knowing the HD model in advance.} 
\label{fig:double_merge_cumulative_learning}
\end{figure*}

\textbf{Lane merging scenario.} The lane merging scenario is depicted in Figure~\ref{fig:double_merge_scenario}. There are two agents (AGENT-0 and AGENT-1) whose goal is to maximize progress (i.e., to reach a goal position at the end of the road), while AGENT-0's goal is also to merge into AGENT-1's lane within a horizon of $H=150$ steps. 
In addition, a HD vehicle (grey car) drives on its lane with a random initial speed. We expect that depending on its speed, the agents coordinate to either overtake the HD vehicle and merge, or drive behind it while allowing AGENT-0 to merge.

\textbf{Intersection scenario.} The intersection scenario is depicted in Figure~\ref{fig:intersection_scenario}. AGENT-0 (bottom part of the figure) wants to turn right, while AGENT-1's goal is to proceed straight on its lane. The HD vehicle has a random initial speed and wants to cross the intersection. Horizon is of $H=150$ steps.\looseness=-1

First, we demonstrate the superior performance of modeling the problem via MA-RL, as opposed to optimizing single-agent policies. We consider the lane merging scenario and assume the HD model is known. Then, we i) compute a game equilibrium $\bm{\mathcal{P}}_\star$ as outlined above and ii) optimize single agent policies $\pi_\star^i$. To obtain $\pi_\star^i$, we consider a single-agent environment for agent~$i$, replacing the other agent with a HD vehicle driving on the same lane. 
We expect ii) to produce inferior policies for the agents, since the learned policies neglect the presence of other AVs and thus miss any opportunity for coordination.
In Figure~\ref{fig:equilibrium_benefit} we report agents' reward from 50 scenario evaluations when the used policies come from either of the two approaches. Equilibrium policies lead to higher average and individual rewards for the agents, consistent with our intuition. In particular, the optimal single-agent policy for AGENT-1 often consists of breaking and driving behind the HD vehicle. Instead, we observed that under equilibrium policies both agents learn better coordination maneuvers allowing AGENT-1 to more often overtake (see Appendix~\ref{app:appendix_experiments} for additional details on this).\looseness=-1

\begin{table}[t!]
\small
  \begin{center}
    \begin{tabular}{r|c|c}
     & \shortstack{Avg. completion rate \\ during learning} & \shortstack{Avg. completion time\\ during learning}\\[0.2em]
     \hline
     \rule{0pt}{1.0\normalbaselineskip}
     pred. mean & 72.0 $\%$ & 8.90 s \\[0.3em]
      TS & 69.9 $\%$ & 8.87 s
      \\[0.3em]
      \textsc{H-MARL} &  \textbf{80.9} $\%$ &  \textbf{8.66} s\\
      \hline
    \end{tabular}
        \vspace{0.4em}

    (a) Lane merging scenario
    
    \vspace{0.6em}
    
        \begin{tabular}{r|c|c}
 & \hphantom{Avg. completion rate} & \hphantom{Avg. distance to goal}\\[-0.7em]
     \hline
          \rule{0pt}{1.0\normalbaselineskip}
      pred. mean & 88.8 $\%$& 9.22 s \\[0.3em]
      TS & 87.4 $\%$& 9.04 s
      \\[0.3em]
      \textsc{H-MARL} &  \textbf{91.9} $\%$ & \textbf{9.02} s\\
      \hline
    \end{tabular}
            \vspace{0.7em}

    (b) Intersection scenario
  \end{center}
  \vspace{-0.9em}
          \caption{Average (across all learning rounds) of the agents' mission completion rates (higher is better) and episodes' completion time (lower is better), when using ``pred. mean'', TS, or \textsc{H-MARL}.}
           \label{table_1}
          \vspace{-1.2em}
\end{table}

Motivated by the above considerations, we evaluate the proposed $\textsc{H-MARL}$ approach in computing equilibrium policies at each round.
To evaluate the performance of a joint distribution $\bm{\mathcal{P}}_t$, we use the average agents' cumulative reward  $\bar{V}(\bm{\mathcal{P}}_t) = \mathbb{E}_{\bm{\pi} \sim \bm{\mathcal{P}}_t} \frac{1}{N}\sum_{i=1}^N \sum_{h=0}^{H-1}r^i(s_h,\bm{\pi}(s_h))$, since we expect this to be high at equilibria (\citet{zanardi2021urban} showed that this is typical of driving games under some assumptions).

In Figure~\ref{fig:double_merge_cumulative_learning}, we plot the average game values $\frac{1}{T}\sum_{t=1}^T \bar{V}(\bm{\mathcal{P}}_t)$ as a function of the interaction rounds $T$, when using \textsc{H-MARL}, Thompson Sampling (TS) exploration, or the predictive posterior mean to compute the agents' policies at each round (we denote this approach as ``pred. mean"), respectively. After $\approx 30$ interaction rounds (corresponding to $\approx 3000$ observed input-output data points to learn the HD model), \textsc{H-MARL} returns equilibrium policies that have comparable performance to knowing the exact HD model in advance (see Figure~\ref{fig:double_merge_overtake} in Appendix~\ref{app:appendix_experiments} for an illustration of successful merging and crossing maneuvers that result from our approach). Moreover, it displays consistently faster learning curves with respect to the
``pred. mean'' and TS baselines. This is due to the fact that \textsc{H-MARL} encourages the AVs to optimistically explore different HD vehicle's behaviors, i.e., the diverse ways the HD vehicle can change its speed based on their position. Instead, the ``pred. mean'' and TS baselines learn about it only indirectly, by greedily choosing the best policies according to the posterior beliefs coming from past observed trajectories. Further differences between the three approaches can also be observed in Table~\ref{table_1}. There, we report the averaged completion rates (i.e., when goal positions are reached) and the average episodes' completion time (faster goal reaching corresponds to higher rewards for the agents) experienced by the agents across all the learning rounds. The policies computed by \textsc{H-MARL} lead to higher completion rates and lower completion times overall.\looseness=-1

In Appendix~\ref{app:appendix_experiments}, we further compare our model-based approach with \emph{model-free} methods. The latter neglect the underlying problem structure, requiring a significantly larger number of environment interactions to achieve comparable rewards.\looseness=-1

\section{Conclusions}
We have considered a model-based multi-agent reinforcement learning problem, where the environment transition function is unknown and can only be learned by costly interactions with the environment. We have proposed \textsc{H-MARL} (Hallucinated Multi-Agent Reinforcement Learning), a novel sample-efficient algorithm that can provably balance exploration with exploitation. \textsc{H-MARL} constructs statistical confidence bounds around the unknown transition function and uses them to build a hallucinated optimistic game for the agents. We have theoretically analyzed our approach by bounding the agents' dynamic regret and deriving a sufficient number of iterations to converge to approximate equilibria of the underlying Markov game. To the best of our knowledge, ours are the first guarantees in general-sum Markov games with continuous states and actions. We have demonstrated our approach on an autonomous driving simulation benchmark, where it showed fast convergence and outperformed non-optimistic and model-free methods. 

\section*{Acknowledgments}
This project received funding from the Swiss National Science Foundation, under the grant SNSF $200021$\textunderscore$172781$ and the NCCR Automation grant 5$1$NF$40$ $180545$, and by the European Union's ERC grant $815943$.

\bibliography{biblio}
\bibliographystyle{icml2021}

\appendix
\onecolumn

\section{Proof of Theorems~\ref{thm:thm1} and \ref{thm:thm2} }\label{app:appendix_proof_thm}
In this section, we prove Theorems~\ref{thm:thm1} and \ref{thm:thm2}. Their proofs strongly rely on two main lemmas, which we articulate next.

\subsection{Confidence Lemma}
The following main lemma shows that if the model is calibrated, the functions computed in Eq.~\eqref{eq:UCB_Vf_modelbased} are indeed an upper confidence bound on the value functions of each agent~$i$. 
\begin{lem}\label{lem:Vf_confidence_lemma}
Under Assumption~\ref{ass:calibrated}, with probability at least $1-\delta$ we have that for all $t\geq 0$,
\begin{equation*}
    \text{UCB}^i_t(\bm{\pi}) \geq V^i(\bm{\pi}) \quad \forall i \in [N], \quad \forall \bm{\pi} \in \bm{\Pi},  \quad \forall t\geq0\,.
\end{equation*}
\end{lem}
\begin{proof}
Under Assumption~\ref{ass:calibrated}, we know that with probability $1-\delta$, for each $s,\mathbf{a}$, and $t$, $|f(s,\mathbf{a}) - \mu_{t-1}(s,\mathbf{a})|\leq \beta_t \sigma_t(s, \mathbf{a})$ holds coordinate-wise. Hence, there exists a $\eta(s,\mathbf{a}) \in [-1,1]^p$ such that $f(s,\mathbf{a}) = \mu_{t-1}(s,\mathbf{a}) + \beta_t \cdot \Sigma_t(s,\mathbf{a})\eta(s,\mathbf{a})$. That is, the hallucinated transition coincides with the true transition. Therefore, given agent~$i$ and joint policy $\bm{\pi}$, the true states' trajectory (evolving according to $f(\cdot)$) is a feasible solution to \eqref{eq:UCB_Vf_modelbased}. 

\end{proof}
\subsection{Bounding optimistic trajectories via Lipschitzness}

First, the following fact shows that under Assumption~\ref{ass:lipschitzness}, the closed-loop transition, reward, and confidence functions are also Lipschitz continuous.

\begin{fact}\label{fact:lipschitzness_cl}
Under the Lipschizness assumption, Assumption~\ref{ass:lipschitzness}, it holds: 
\begin{align*}
\| f(s,\bm{\pi}(s)) - f(s',\bm{\pi}(s')) \|_2  \leq L_f \sqrt{1 + N L_\pi} \cdot \|s-s'\|_2 , \\ 
         \| r^i(s,\bm{\pi}(s)) - r^i(s',\bm{\pi}(s')) \|_2  \leq L_r \sqrt{1 + N L_\pi} \cdot \|s-s'\|_2 \,, \quad  \forall i ,\\
          \| \sigma_t(s,\bm{\pi}(s)) - \sigma_t(s',\bm{\pi}(s')) \|_2  \leq L_\sigma \sqrt{1 + N L_\pi} \cdot \|s-s'\|_2 \,, \quad  \forall t \,.
\end{align*}
\end{fact}
\begin{proof}
Using the Lipschitzness of $f(\cdot)$ and $\pi^i(\cdot)$ from Assumption~\ref{ass:lipschitzness}, we have:
\begin{align*}
    \| f(s,\bm{\pi}(s)) - f(s',\bm{\pi}(s')) \|_2 & \leq L_f \|(s-s',\bm{\pi}(s)-\bm{\pi}(s') )\|_2  \\
    & = L_f \sqrt{\|s-s'\|_2^2 + \sum_{i=1}^N\|\pi^i(s)-\pi^i(s') \|_2^2}\\
    & \leq  L_f \sqrt{\|s-s'\|_2^2 + N L_\pi\|s-s' \|_2^2}\\ 
    & = L_f \sqrt{1 + N L_\pi} \cdot \|s-s'\|_2.
    \end{align*}
The same derivation is obtained for the functions $r^i$ and $\sigma_t$ using Lipschitz constants $L_r$ and $L_\sigma$, respectively.
\end{proof}

Then, we use the above Lipschitz properties to bound the distance between the optimistic value functions $\text{UCB}^i_t(\bm{\pi})$ and the true functions $V^i(\bm{\pi})$. This will depend on the sum of accumulated standard deviations, as stated in the following main lemma. 
\begin{lem}\label{lem:diff_ucb_Vf}
Define $\bar{L}_f = 1 + (L_f+ 2 \beta_{t-1}L_\sigma)\sqrt{1 + NL_\pi}$ and consider any round $t$. For each agent~$i$ and joint policy $\bm{\pi}$, under Assumption~\ref{ass:calibrated} it holds
\begin{align}
    \left|\text{UCB}^i_{t-1}(\bm{\pi}) - V^i(\bm{\pi})\right| \leq 2\beta_{t-1} L_r \sqrt{1 + N L_\pi}\bar{L}_f^{H-1} H   \cdot \mathbb{E}_{\bm{\omega}}\left[\sum_{h=0}^{H-1} \|\sigma_{t-1}(s_h, \bm{\pi}(s_h))\|_2\right],
\end{align}
    where $\{s_h\}_{h=0}^{H-1}$ is the sequence of environment's states when agents play according to $\bm{\pi}$ and $\bm{\omega}$ is the vector of noise realizations $\bm{\omega} = [\omega_0, \ldots, \omega_{H-1}]$.
\end{lem}
\begin{proof}
Lemma~\ref{lem:diff_ucb_Vf} can be obtained from \citep[Lemmas~3,4]{curi2020} as follows. 

For a given joint policy $\bm{\pi} \in \bm{\Pi}$, let $\{s_h\}_{h=0}^{H-1}$ be the sequence of environment states when agents play according to $\bm{\pi}$ and the environment transition function is $f(\cdot)$. Note that this sequence is random, depending on the noise realization vector $\bm{\omega}= [w_0,\ldots, w_{H-1}]$. According to this notation, the value function of a generic agent~$i$ can be written as: 
\begin{equation*}
    V^i(\bm{\pi}) = \mathbb{E}_{\bm{\omega}}\left[ \sum_{h=0}^{H-1} r^i(s_h, \bm{\pi}(s_h)) 
   \right].
\end{equation*}
Similarly, consider an hallucinated transition function $\tilde{f}(\cdot)$ and let $\{\tilde{s}_h\}_{h=0}^{H-1}$ be the sequence of environment states visited according to $\tilde{f}(\cdot)$. Also this sequence is random, depending on the noise realizations which we denote with $\tilde{\bm{\omega}}$.
Consider now the hallucinated function $\tilde{f}(\cdot) := \mu_{t-1}(\cdot) + \beta_{t-1}\cdot \Sigma_{t-1}(\cdot) \eta^\star(\cdot)$, where $\eta^\star$ is the auxiliary function that maximizes Eq.~\eqref{eq:UCB_Vf_modelbased} at round $t-1$. According to the introduced notation, we have 
\begin{equation*}
    \text{UCB}_{t-1}^i(\bm{\pi}) = \mathbb{E}_{\tilde{\bm{\omega}}}\left[ \sum_{h=0}^{H-1} r^i(\tilde{s}_h, \bm{\pi}(\tilde{s}_h)) 
   \right].
\end{equation*}

Now, we can use the Lipschitz properties of the closed-loop reward functions (Fact~\ref{fact:lipschitzness_cl}) to obtain
\begin{align}
     \left|\text{UCB}^i_{t-1}(\bm{\pi}) - V^i(\bm{\pi})\right| & = \left|\mathbb{E}_{\tilde{\bm{\omega}}}\left[ \sum_{h=0}^{H-1} r^i(\tilde{s}_h, \bm{\pi}(\tilde{s}_h)) 
   \right] - \mathbb{E}_{\bm{\omega}}\left[ \sum_{h=0}^{H-1} r^i(s_h, \bm{\pi}(s_h)) 
   \right] \right|  \nonumber \\ 
   & = \left|\mathbb{E}_{\tilde{\bm{\omega}}= \bm{\omega} }\left[ \sum_{h=0}^{H-1} r^i(\tilde{s}_h, \bm{\pi}(\tilde{s}_h)) 
   - r^i(s_h, \bm{\pi}(s_h)) 
   \right] \right| \nonumber \\ 
   & \leq L_r \sqrt{1 + N L_\pi} \sum_{h=0}^{H-1} \mathbb{E}_{\tilde{\bm{\omega}}= \bm{\omega}}\big[\|s_{h} - \tilde{s}_{h}\|_2 \big], \label{eq:proof_lemma2_1}
\end{align}
where $\mathbb{E}_{\tilde{\bm{\omega}}= \bm{\omega}}$ is the expectation over $\bm{\omega}$ and taking $\tilde{\bm{\omega}}= \bm{\omega}$. 

At this point, we are left to bound the accumulated difference between the true and the hallucinated trajectory: $\sum_{h=0}^{H-1} \mathbb{E}_{\tilde{\bm{\omega}}= \bm{\omega}}\big[\|s_{h} - \tilde{s}_{h}\|_2 \big]$. For this we can invoke \citep[Lemma 4]{curi2020} which, via a sequence of induction steps and under the calibrated model Assumption~\ref{ass:calibrated}, shows that
\begin{equation*}
        \|s_{h} - \tilde{s}_{h}\|_2 \leq 2 \beta_{t-1}\bar{L}_f^{H-1}\sum_{\tau=0}^{h-1} \|\sigma_{t-1}(s_{\tau}, \bm{\pi}(s_\tau)) \|_2.
\end{equation*}

Then, we can substitute in Eq.~\eqref{eq:proof_lemma2_1} the bound above to obtain
\begin{align*}
    |\text{UCB}^i_{t-1}(\bm{\pi}) - V^i(\bm{\pi})|
   & \leq 2 \beta_{t-1} L_r \sqrt{1 + N L_\pi} \bar{L}_f^{H-1} \sum_{h=0}^{H-1}  \mathbb{E}_{\bm{\omega}}\left[\sum_{\tau=0}^{h-1} \|\sigma_{t-1}(s_{\tau}, \bm{\pi}(s_\tau)) \|_2\right] \\ 
   &  \leq 2 \beta_{t-1} L_r \sqrt{1 + N L_\pi} \bar{L}_f^{H-1} H \cdot \mathbb{E}_{\bm{\omega}}\left[\sum_{h=0}^{H-1} \|\sigma_{t-1}(s_{h}, \bm{\pi}(s_h)) \|_2\right], 
\end{align*}
which completes the proof.
\end{proof}

\subsection{Proof of Theorem~\ref{thm:thm1}}
We are now ready to prove Theorem~\ref{thm:thm1}.
\begin{proof}
According to \textsc{H-MARL}, the joint policy $\bm{\pi}_t$ is sampled from distribution $\bm{\mathcal{P}}_t$ which, by construction, is a CCE of the game defined by the optimistic value functions $\text{UCB}^1_{t-1}(\cdot), \ldots, \text{UCB}^N_{t-1}(\cdot)$ (see Line~2 in Algorithm~\ref{alg:H_MARL}). Hence, by definition of CCE (Def.~\ref{def:CCE}), 
for each player $i$ and any policy $\pi^i \in \Pi^i$, 
\begin{equation}\label{eq:proof_thm1_cce}
    \mathbb{E}_{\bm{\pi} \sim \bm{\mathcal{P}}_t}[\text{UCB}_{t-1}^i(\bm{\pi})] \geq \mathbb{E}_{\pi^{-i} \sim \bm{\mathcal{P}}_t^{-i}}[\text{UCB}_{t-1}^i(\pi^i, \pi^{-i})] \,.
\end{equation}
Equation~\eqref{eq:proof_thm1_cce} implies that, for each player $i$ and and any policy $\pi^i \in \Pi^i$, under Assumption~\ref{ass:calibrated} with probability at least $1-\delta$,
\begin{align}
  \mathbb{E}_{\bm{\pi} \sim \bm{\mathcal{P}}_t}[V^i(\bm{\pi})] & = \mathbb{E}_{\bm{\pi} \sim \bm{\mathcal{P}}_t}[\text{UCB}^i_{t-1}(\bm{\pi})] -  \mathbb{E}_{\bm{\pi} \sim \bm{\mathcal{P}}_t}\big[\text{UCB}^i_{t-1}(\bm{\pi}) - V^i(\bm{\pi})\big]  \nonumber \\ 
  & \geq \mathbb{E}_{\pi^{-i} \sim \bm{\mathcal{P}}_t^{-i}}[\text{UCB}^i_{t-1}(\pi^i, \pi^{-i})] -  \mathbb{E}_{\bm{\pi} \sim \bm{\mathcal{P}}_t}\big[\text{UCB}^i_{t-1}(\bm{\pi}) - V^i(\bm{\pi})\big]  \nonumber \\ 
  & \geq \mathbb{E}_{\pi^{-i} \sim \bm{\mathcal{P}}_t^{-i}}[V^i(\pi^i, \pi^{-i})]  -  \mathbb{E}_{\bm{\pi} \sim \bm{\mathcal{P}}_t}\big[\text{UCB}^i_{t-1}(\bm{\pi}) - V^i(\bm{\pi})\big] , \label{eq:thm1_proof_aux_1}
\end{align}
where the second inequality follows from \eqref{eq:proof_thm1_cce}, and the third one from the model being calibrated (Assumption~\ref{ass:calibrated}) and Lemma~\ref{lem:Vf_confidence_lemma}. Using this, we can bound the dynamic regret for agent~$i$ as:
\begin{align}
    R^i(T) & :=   \sum_{t=1}^{T} \max_{\pi \in \Pi^i} \mathbb{E}_{\pi^{-i} \sim \bm{\mathcal{P}}_t^{-i}}\left[ V^i(\pi, \pi^{-i})\right] -  \mathbb{E}_{\bm{\pi} \sim \bm{\mathcal{P}}_t}\left[V^i({\bm\pi})\right] \label{eq:derivation_regret} \\ 
    & \leq \sum_{t=1}^{T}  \mathbb{E}_{\bm{\pi} \sim \bm{\mathcal{P}}_t}\big[\text{UCB}^i_{t-1}(\bm{\pi}) - V^i(\bm{\pi})\big] \nonumber
    \\
    & \leq \sum_{t=1}^{T} \underbrace{2\beta_{t-1} L_r \sqrt{1 + N L_\pi}\bar{L}_f^{H-1}}_{\leq \bar{L}} H   \cdot \mathbb{E}_{\bm{\pi} \sim \bm{\mathcal{P}}_t, \bm{w}}\left[\sum_{h=0}^{H-1} \|\sigma_{t-1}(s_h, \bm{\pi}(s_h))\|_2\right] \label{eq:proof_thm1_bound_needed} \\ 
    & \leq \bar{L} H \sum_{t=1}^{T}\mathbb{E}_{\bm{\pi} \sim \bm{\mathcal{P}}_t, \bm{w}}\left[\sum_{h=0}^{H-1} \|\sigma_{t-1}(s_h, \bm{\pi}(s_h))\|_2\right] \nonumber \\ 
    & \leq \bar{L} H \sqrt{TH\sum_{t=1}^{T}\mathbb{E}_{\bm{\pi} \sim \bm{\mathcal{P}}_t, \bm{w}}\left[\sum_{h=0}^{H-1} \|\sigma_{t-1}(s_h, \bm{\pi}(s_h))\|_2^2\right]}  \nonumber
\end{align}
where $\bar{L} = 2\bar{\beta} L_r \sqrt{1 + N L_\pi} \left(1 + (L_f+ 2 \bar{\beta}L_\sigma)\sqrt{1 + NL_\pi}\right)^{H-1}$. The first inequality follows from Eq.~\eqref{eq:thm1_proof_aux_1}. The second one follows from Lemma~\ref{lem:diff_ucb_Vf}, the third one by definition of $\bar{L}$ and $\beta_t \leq \bar{\beta}$ for all $t$, while the fourth one by Cauchy-Schwartz. 

Now, applying standard concentration arguments (see, e.g., \citep[Lemma 3]{kirschner2018information}), with probability at least $1-\delta$, it holds
\begin{equation}\label{eq:proof_thm1_concentration}
    \sum_{t=1}^{T}\mathbb{E}_{\bm{\pi} \sim \bm{\mathcal{P}}_t, \bm{w}}\left[\sum_{h=0}^{H-1} \|\sigma_{t-1}(s_h, \bm{\pi}(s_h))\|_2^2\right] \leq \mathcal{O}\left( \sum_{t=1}^{T} \sum_{(s, \mathbf{a}) \in \mathcal{D}_t} \|\sigma_{t-1}(s, \mathbf{a}) \|_2^2 + \log(\frac{1}{\delta}) \right).
\end{equation}
This is because at each round $t$, the term $\mathbb{E}_{\bm{\pi} \sim \bm{\mathcal{P}}_t, \bm{w}}\left[\sum_{h=0}^{H-1} \|\sigma_{t-1}(s_h, \bm{\pi}(s_h))\|_2^2\right]$ is the expected value of $ \sum_{(s, \mathbf{a}) \in \mathcal{D}_t} \|\sigma_{t-1}(s, \mathbf{a}) \|_2^2$, given the history up to round $t-1$.

Hence, by taking the union bound over the events of Lemma~\ref{lem:Vf_confidence_lemma} and Eq.~\eqref{eq:proof_thm1_concentration}, the regret of each agent~$i$ is bounded, with probability at least $1-\delta$ by:
\begin{align*}
    R^i(T) & \leq \mathcal{O} \left( \bar{L} H \sqrt{TH\sum_{t=1}^{T}\sum_{(s, \mathbf{a}) \in \mathcal{D}_t} \|\sigma_{t-1}(s, \mathbf{a}) \|_2^2  + \log(1/\delta)}  \right) \nonumber \\
    & \leq \mathcal{O}\big( \bar{L} H \sqrt{TH \mathcal{I}_T}\big), 
\end{align*}
since $\sum_{t=1}^{T}\sum_{(s, \mathbf{a}) \in \mathcal{D}_t} \|\sigma_{t-1}(s, \mathbf{a}) \|_2^2 \leq \mathcal{I}_T$ by definition of $\mathcal{I}_T$.
\end{proof}
\subsection{When only approximate CCEs are computed at each round}

Theorem~\ref{thm:thm1} assumes that an exact CCE is computed at each round of \textsc{H-MARL} (Line 2 of Algorithm~\ref{alg:H_MARL}). However, in practice one may obtain only $\epsilon_t$-CCE at each round $t$. The next theorem shows that in such a case, the agents' dynamic regret suffers an additive factor which corresponds to the sum of approximation errors $\epsilon_t$. 
As a result, even if equilibria are not computed exactly, agents' dynamic regret can still be sublinear provided that $\epsilon_t$ decreases sufficiently fast.

\begin{thm}\label{thm:approximate}
Let Assumptions~\ref{ass:calibrated},\ref{ass:lipschitzness} be satisfied. Moreover, consider the case where \textsc{H-MARL} computes a $\epsilon_t$-CCE at each round (Line 2 of Algorithm~\ref{alg:H_MARL}). After $T$ rounds, with probability $1-\delta$, each agent~$i$ has bounded dynamic regret:
\begin{equation*}
R^i(T)  \leq  \bar{L} H^{1.5} \sqrt{T \mathcal{I}_T}+ \sum_{t=1}^T \epsilon_t
\end{equation*}
where $\bar{L}=\mathcal{O}\big(N^{H/2} L_\pi^{H/2} (\bar{\beta}^H L_\sigma^H + L_f^H) + \log(1/\delta)\big)$, $\bar{\beta} = \max_{t} \beta_t$, and $\mathcal{I}_T$ is the complexity measure defined in \eqref{eq:complexity_measure}.
\end{thm}

\begin{proof}
At each round $t$, \textsc{H-MARL} computes an $\epsilon_t$-CCE of the game associated to value functions $\text{UCB}^1_{t-1}(\cdot), \ldots, \text{UCB}^N_{t-1}(\cdot)$. Hence,
for each player $i$ and any policy $\pi^i \in \Pi^i$, 
\begin{equation*}
    \mathbb{E}_{\bm{\pi} \sim \bm{\mathcal{P}}_t}[\text{UCB}_{t-1}^i(\bm{\pi})] \geq \mathbb{E}_{\pi^{-i} \sim \bm{\mathcal{P}}_t^{-i}}[\text{UCB}_{t-1}^i(\pi^i, \pi^{-i})] - \epsilon_t \,.
\end{equation*}
By following the same steps of proof of Theorem~\ref{thm:thm1}, this implies that
\begin{equation}\label{eq:proof_thm_approx}
  \mathbb{E}_{\bm{\pi} \sim \bm{\mathcal{P}}_t}[V^i(\bm{\pi})] 
 \geq \mathbb{E}_{\pi^{-i} \sim \bm{\mathcal{P}}_t^{-i}}[V^i(\pi^i, \pi^{-i})]  - \epsilon_t -  \mathbb{E}_{\bm{\pi} \sim \bm{\mathcal{P}}_t}\big[\text{UCB}^i_{t-1}(\bm{\pi}) - V^i(\bm{\pi})\big] ,
\end{equation}
where we have used Assumption~\ref{ass:calibrated} and Lemma~\ref{lem:Vf_confidence_lemma}. Then, we can use Eq.~\eqref{eq:proof_thm_approx} to bound agents' regrets obtaining, 
\begin{align*}
    R^i(T) & :=  \sum_{t=1}^{T} \max_{\pi \in \Pi^i} \mathbb{E}_{\pi^{-i} \sim \bm{\mathcal{P}}_t^{-i}}\left[ V^i(\pi, \pi^{-i})\right] -  \mathbb{E}_{\bm{\pi} \sim \bm{\mathcal{P}}_t}\left[V^i({\bm\pi})\right] \\ 
    & \leq   \sum_{t=1}^T \epsilon_t + \bar{L} H \sum_{t=1}^{T}\mathbb{E}_{\bm{\pi} \sim \bm{\mathcal{P}}_t, \bm{w}}\left[\sum_{h=0}^{H-1} \|\sigma_{t-1}(s_h, \bm{\pi}(s_h))\|_2\right] \\ 
    & \leq \bar{L} H \sqrt{TH \mathcal{I}_T} + \sum_{t=1}^T \epsilon_t, 
\end{align*}
where we have used the same derivation done after Eq.~\eqref{eq:derivation_regret}.
\end{proof}

\subsection{Proof of Theorem~\ref{thm:thm2}}
\begin{proof}
At round $t^\star$, where $t^\star$ is selected according to Theorem~\ref{thm:thm2}, the proof steps of Theorem~\ref{thm:thm1} (in particular Eq.~\eqref{eq:thm1_proof_aux_1}) imply that 
\begin{equation}
\mathbb{E}_{\bm{\pi} \sim \bm{\mathcal{P}}_t^\star}[V^i(\bm{\pi})] \geq \mathbb{E}_{\pi^{-i} \sim {\bm{\mathcal{P}}_t^\star}^{-i}}[V^i(\pi^i, \pi^{-i})]  -  \underbrace{ \mathbb{E}_{\bm{\pi} \sim \bm{\mathcal{P}}_t^\star}\big[\text{UCB}^i_{t^\star-1}(\bm{\pi}) - V^i(\bm{\pi})\big]}_{:= \epsilon_T^i} ,
\end{equation}
That is, distribution $\bm{\mathcal{P}}_t^\star$ is an $\epsilon$-CCE of the underlying Markov game, where the approximation factor $\epsilon$ is bounded by $\max_{i} \epsilon_T^i$ and $\epsilon_T^i$ is the quantity defined above. In what follows, we show that 
\begin{equation}\label{eq:proof_thm2_bound}
    \max_{i} \epsilon_T^i \leq   2 T^{-\frac{1}{2}}\bar{L} H^{1.5} \sqrt{\mathcal{I}_T}.
\end{equation}
Then, Theorem~\ref{thm:thm2} follows by selecting an accuracy level $\epsilon>0$ and use Eq.~\eqref{eq:proof_thm2_bound} to solve for $T$.

To prove \eqref{eq:proof_thm2_bound}, recall that $t^\star$ is selected according to: 
\begin{equation}\label{eq:proof_thm2_tstar}
    t^\star = \arg \min_{t \in [T]} \max_{i\in [N]} \mathop{\mathbb{E}}_{\bm{\pi} \sim \bm{\mathcal{P}}_{t}}[ \textup{UCB}^i_{t-1}(\bm{\pi})-\textup{LCB}^i_{t-1}(\bm{\pi})] , 
\end{equation}
where the function 
$\textup{LCB}^i_t(\bm{\pi})$ is the solution of Eq.~\eqref{eq:UCB_Vf_modelbased} with the outer maximization replaced by a minimization over $\eta(\cdot)$. 

By following the same steps of the confidence Lemma~\ref{lem:Vf_confidence_lemma}, with probability at least $1-\delta$ it holds 
\begin{equation}\label{eq:proof_thm2_confidence}
\text{LCB}^i_t(\bm{\pi}) \leq V^i(\bm{\pi}) \quad \forall i \in [N], \quad \forall \bm{\pi} \in \bm{\Pi},  \quad \forall t\geq0 ,   
\end{equation}
that is, $\textup{LCB}^i_t(\bm{\pi})$ is a lower confidence bound on the agents' value functions.
Moreover, the distance $\left|\text{LCB}^i_t(\bm{\pi})- V^i(\bm{\pi})\right|$ can be bounded exactly as was done in Lemma~\ref{lem:diff_ucb_Vf} for $\left|\text{UCB}^i_t(\bm{\pi})- V^i(\bm{\pi})\right|$. The only difference in its proof consists of considering $\{\tilde{s}_h\}_{h=0}^{H-1}$ to be the sequence of environment states resulting from the pessimistic transition model $\tilde{f}(\cdot) = \mu_{t-1}(\cdot) + \beta_{t-1}\cdot \Sigma_{t-1}(\cdot)\eta^{lcb}(\cdot)$ where $\eta^{lcb}$ is the minimizer of Eq.~\eqref{eq:UCB_Vf_modelbased}.

Then, according to \eqref{eq:proof_thm2_tstar} and \eqref{eq:proof_thm2_confidence} we can bound $\max_i \epsilon_T^i$ as
\begin{align*}
    \max_i \epsilon_T^i & := \max_i \mathbb{E}_{\bm{\pi} \sim \bm{\mathcal{P}}_t^\star}\big[\text{UCB}^i_{t^\star-1}(\bm{\pi}) - V^i(\bm{\pi})\big] \\ 
    & \leq  \max_i \mathbb{E}_{\bm{\pi} \sim \bm{\mathcal{P}}_t^\star}\big[\text{UCB}^i_{t^\star-1}(\bm{\pi}) - \text{LCB}^i_{t^\star-1}(\bm{\pi})\big]
    \\ 
    & \leq  \frac{1}{T} \sum_{t=1}^T \max_i  \mathbb{E}_{\bm{\pi} \sim \bm{\mathcal{P}}_t}\big[\text{UCB}^i_{t-1}(\bm{\pi}) - \text{LCB}^i_{t-1}(\bm{\pi})\big] \\ 
    & \leq  \frac{1}{T} \sum_{t=1}^T \max_i  \mathbb{E}_{\bm{\pi} \sim \bm{\mathcal{P}}_t}\big[\left|\text{UCB}^i_{t-1}(\bm{\pi}) - V^i(\bm{\pi}) \right| + \left|\text{LCB}^i_{t-1}(\bm{\pi}) - V^i(\bm{\pi}) \right| \big]
    \\
    & \leq \frac{2}{T} \sum_{t=1}^{T} 2\beta_{t-1} L_r \sqrt{1 + N L_\pi}\bar{L}_f^{H-1} H   \cdot \mathbb{E}_{\bm{\pi}_t, \bm{w}}\left[\sum_{h=0}^{H-1} \|\sigma_{t-1}(s_h, \bm{\pi}_t(s_h))\|_2\right] \\
    & \leq \frac{2}{T} \bar{L} H \sqrt{TH \mathcal{I}_T} = 2 T^{-\frac{1}{2}}\bar{L} H^{1.5} \sqrt{\mathcal{I}_T}.
\end{align*}
In the last equality we have used the same bound from Eq.~\eqref{eq:proof_thm1_bound_needed}.
\end{proof}

\section{Supplementary material for Section~\ref{sec:experiments}}\label{app:appendix_experiments}

In this section, we provide additional details concerning our experimental setup illustrated in Section~\ref{sec:experiments}.

\textbf{Multi-agent equilibria vs. single-agent optimal policies.} We compare the performance of multi-agent equilibria, with respect to computing single-agent optimal policies for the AVs. 
We consider the lane merging scenario and assume the HD vehicles' model is known. Then, we obtain driving policies for the AVs by:
\begin{itemize}
    \item[i)] Computing equilibrium policies using independent DQN learning. AVs' policies are trained simultaneously using the DQN algorithm, for 600 iterations. Then, we select the joint policy leading to the highest sum of agents' reward. 
    \item[ii)] Computing single-agent optimal policies. These are obtained as follows: We consider one agent at a time, by removing the other agent from the environment and replacing it with a HD vehicle driving on the same lane. For each agent, we obtain an optimal policy running DQN algorithm for 600 iterations and selecting the checkpoint with highest reward.
\end{itemize}
We evaluate the policies coming from either i) or ii) for 50 episodes and plot the corresponding rewards in Figure~\ref{fig:equilibrium_benefit}. Equilibrium policies lead to higher average and individual rewards for the agents. In particular, we observe that the single-agent optimal policy of AGENT-1 is often to break and drive behind the HD vehicle, especially when the HD vehicle drives at moderate/high speed. Instead, when using equilibrium policies, both AGENT-0 and AGENT-1 often coordinate in overtaking the HD vehicle, yielding higher rewards. This is confirmed by Table~\ref{table:appendix_equilibrium_advantage} which shows that agent's distance to their goal position (especially for AGENT-1) is lower when using equilibrium policies. 
\begin{table}[h!]
\small
  \begin{center}
    \begin{tabular}{r|c|c|c}
     & Avg. agent & AGENT-0 & AGENT-1\\[0.2em]
     \hline
     \rule{0pt}{1.0\normalbaselineskip}
      Single-agent optima & 57.49 m & 52.43 m & 62.55 m \\[0.3em]
      Multi-agent equilibria &  \textbf{56.09} m &  \textbf{52.17} m &  \textbf{60.02} m\\
      \hline
    \end{tabular}
  \end{center}
  \vspace{-0.9em}
          \caption{Agents' distance to their goal position (lower is better) averaged over 50 runs, when using single-agent optimal policies versus equilibrium policies.}
           \label{table:appendix_equilibrium_advantage}
          \vspace{-1.2em}
\end{table}

\textbf{HD vehicles' model.} 
To learn the behavior of HD vehicles, we use a GP model. The model maps the current HD vehicle's state (position and velocity) and its relative position to the other vehicles, to the next state for the HD vehicle. More specifically, we train our GP model on the \textit{changes} (i.e., increase or decrease) of the HD vehicle's position and velocity. Hence, the next HD state is obtained by summing the current state's coordinates with the GP predictions. 
We use a Màtern kernel for predicting the speed change (as we expect this to be quite nonsmooth), and a squared exponential kernel to predict its change in position.
However, we have observed different kernel choices (e.g., only using squared-exponential kernels) to produce similar performance.  At the end of each round $t$, GP inference is performed on the whole set of past observed trajectories $\{\mathcal{D}_\tau\}_{\tau=1}^t$ using GPyTorch~\cite{gardner2018gpytorch} with Adam ~\citep{kingma2014adam} optimizer for 50 iterations with learning rate $l=0.1$. At round $0$, we initialize the HD model with 2 random samples taken from past simulations.

\textbf{\textsc{H-MARL} implementation.} 
To compute equilibrium policies at each round $t$ (Line 2 of Algorithm~\ref{alg:H_MARL}), we run $50$ iterations of independent DQN~\cite{mnih2015human} learning (initialized by the $50^{th}$ checkpoint from round $t-1$) and select distribution $\bm{\mathcal{P}}_{t+1}$ to be the uniform distribution over checkpoints $\{35^{th}, 40^{th}, 45^{th}, 50^{th}\}$. As mentioned in Section~\ref{sec:experiments}, this alleviates non-convergence behaviors and mimics the way CCEs are computed by independent no-regret learning dynamics. 
The hallucinated optimistic value functions $\text{UCB}_t^i(\cdot)$ are approximated by the sampling approach of Eq.~\eqref{eq:UCB_practical} with $Z=5$ samples at each time step and $\beta_t = 1.0$. This is effectively achieved as follows. During each policy evaluation step, we sample $Z$ plausible next states for the HD vehicle and keep only the one that leads to the largest reward for the agents. Then, this process continues until the end of the episode. We have observed that changes in HD vehicles' position are accurately predicted by the GP model even with few data (this is because, given current position and speed, they only depend on the kinematics of the car and are invariant to the position of the AVs). Hence, for computational efficiency, we always use the posterior mean about the next HD vehicle's position, and sample $Z=5$ plausible values for its next speed. We have also observed that uniformly spaced samples lead to better performance, compared to random i.i.d. samples.   \looseness=-1

\textbf{Thompson Sampling (TS) implementation.} We follow the exact same implementation of \textsc{H-MARL} with the important difference on how HD vehicle's states are computed during policy evaluation. Indeed, according to TS the next HD vehicle's state $s_{h+1}$ is sampled from the posterior Gaussian distribution with mean $\mu_t(s_h, \mathbf{a}_h)$ and covariance $\Sigma_t(s_h, \mathbf{a}_h)$. Note that this is substantially different from the proposed optimistic \textsc{H-MARL} where, among plausible next states, we select the ones leading to the highest agents' reward.

\textbf{Model-based vs. model-free.} In this section, we compare our model-based \textsc{H-MARL} approach with the more general \emph{model-free} DQN~\cite{mnih2015human} algorithm. In DQN, agents are trained directly on the rewards observed from the `true'  environment (i.e., interacting with the true HD vehicle's model). Instead, in 
\textsc{H-MARL} we assume a known model structure and agents interact with a hallucinated version of the environment where the HD vehicles are simulated according to the current (optimistic) model estimate. Indeed, in \textsc{H-MARL} we utilize a DQN subroutine at each round (see our implementation above) which, however, interacts only with the hallucinated model and not with the true one. Hence, we expect the model-free DQN to require a significantly higher number of environment interactions to achieve comparable performance. This is confirmed in Figure~\ref{fig:comparison_with_modelfree}, where we plot agents' average reward as a function of the interaction rounds, both for the lane merging and intersection scenario. Perhaps not surprisingly, the model-free DQN requires interactions of several higher orders of magnitudes to achieve similar rewards.

\begin{figure*}[t!]
\centering
\begin{subfigure}[b]{0.6\textwidth}
\includegraphics[width=\textwidth]{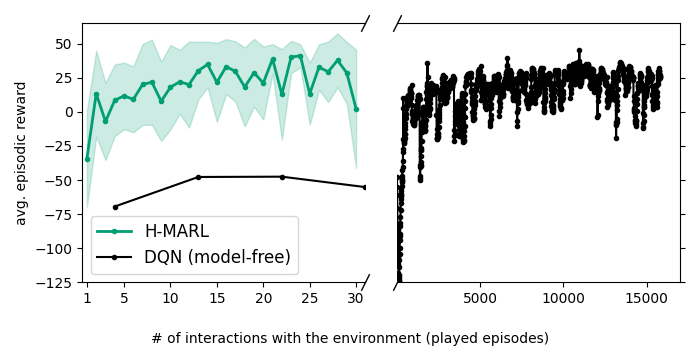}
\caption{Lane merging scenario}
\end{subfigure}
\hspace{1.5em}
\begin{subfigure}[b]{0.6\textwidth}
\includegraphics[width=\textwidth]{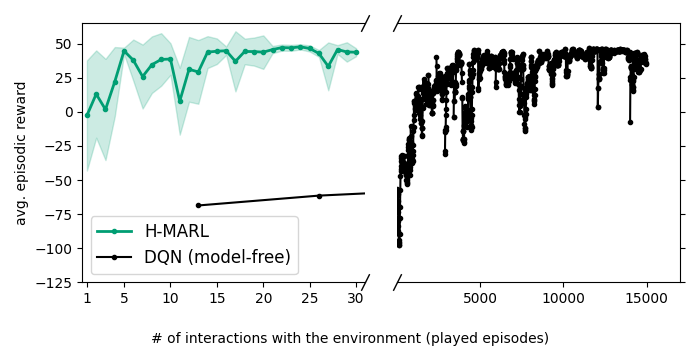}
\caption{Intersection scenario}
\end{subfigure}
\caption{Average agents' reward as a function of the interactions with the true environment (i.e., with the true HD vehicle's model), when agents' policies are computed according to the model-based \textsc{H-MARL}, or when using the model-free DQN~\cite{mnih2013playing} algorithm. Model-free DQN requires a significantly higher number of interactions to achieve comparable performance. We remark that \textsc{H-MARL} utilizes a DQN subroutine at each round which, however, interacts with a hallucinated model for the HD vehicle and not with the true one.} 
\label{fig:comparison_with_modelfree}
\end{figure*}

\begin{figure*}[ht]
\centering
\begin{subfigure}[b]{0.4\textwidth}
\includegraphics[height=0.5\textwidth]{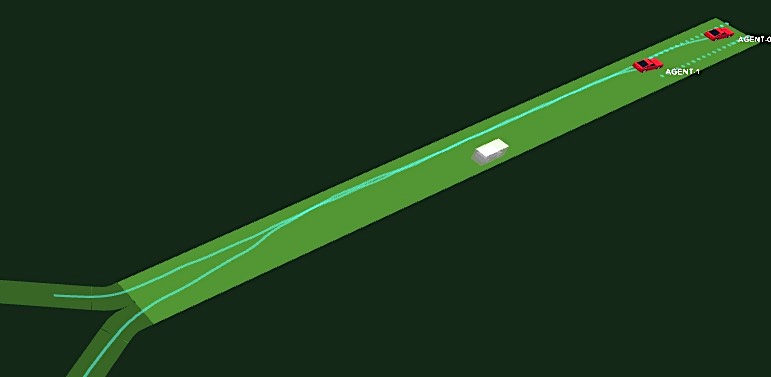}
\caption{Overtake and merge maneuver}
\label{fig:double_merge_overtake}
\end{subfigure}
\hspace{2.5em}
\begin{subfigure}[b]{0.4\textwidth}
\includegraphics[height=0.5\textwidth]{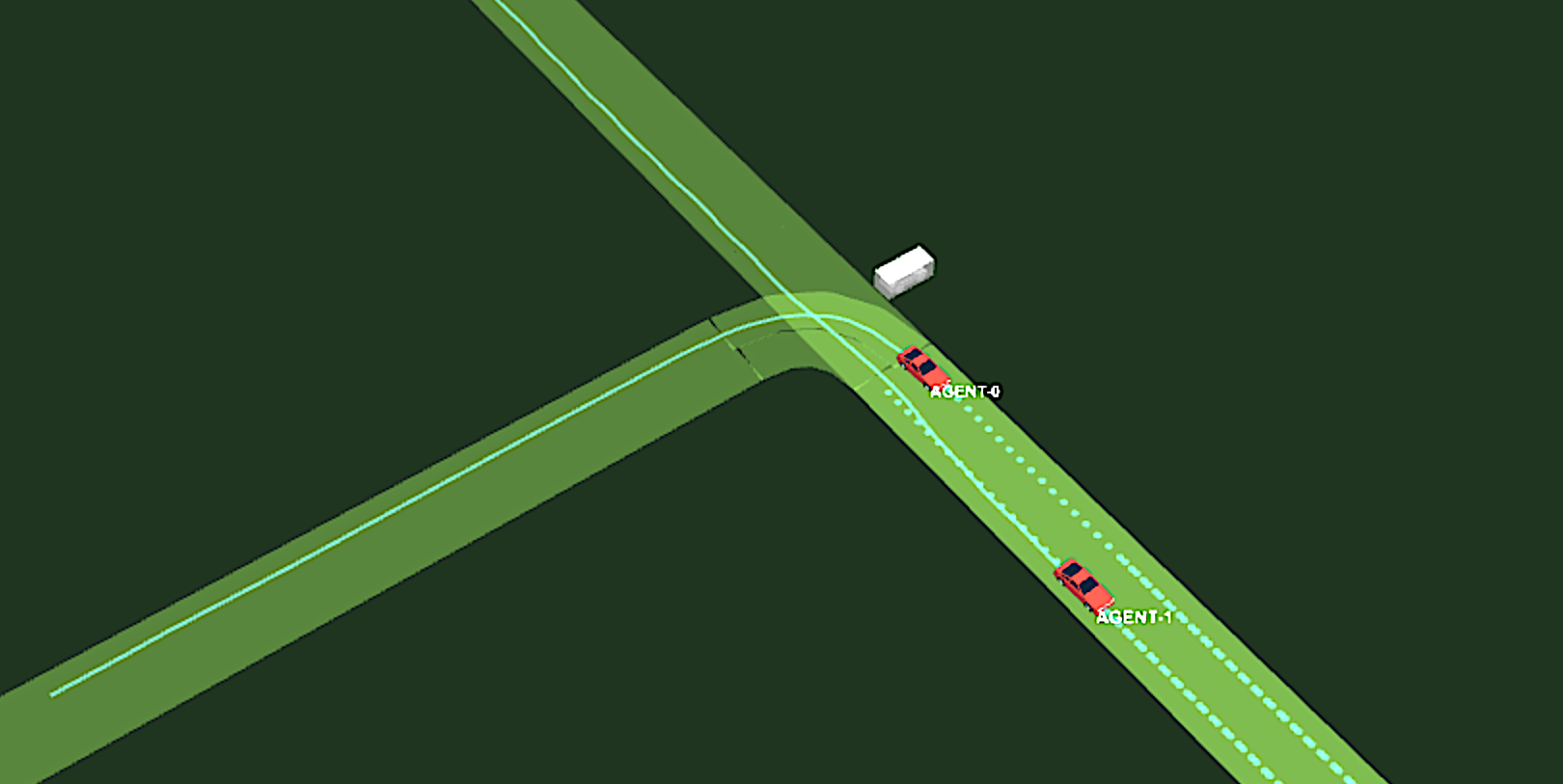}
\caption{Intersection crossing}
\label{fig:intersection_maneuver}
\end{subfigure}
\caption{(a) \textbf{Lane merging scenario} (see Fig.~\ref{fig:scenarios}~(a)). Example of a successful overtake and merge maneuver: Both agents accelerate so that AGENT-$1$ overtakes the HD vehicle, and AGENT-$0$ successfully merges. (b) \textbf{Intersection scenario} (see Fig.~\ref{fig:scenarios}~(b)). Example of a successful crossing: AGENT-1 crosses before both AGENT-0 and the HD vehicle, while AGENT-0 waits for the HD vehicle to cross, and then turns right.}
\label{fig:example_of_maneuvers}
\end{figure*}


\end{document}